\documentclass[12pt]{article}
\pdfoutput=1
\usepackage[utf8]{inputenc}
\usepackage{hyperref}
\usepackage{graphicx}
\usepackage{amsmath}
\usepackage{amsfonts}
\usepackage{amssymb,multirow}
\usepackage{amsthm}
\usepackage{graphicx}

\DeclareMathOperator*{\argmax}{argmax}

\newtheorem{corollary}{Corollary}
\newtheorem{lemma}{Lemma}

\title{Kernel Treelets}
\author{Hedi Xia and Hector D. Ceniceros}
\date{Department of Mathematics, University of California Santa Barbara, 93106}
\begin{document}

\maketitle

\begin{abstract} A new method for hierarchical clustering is presented. It combines treelets, a particular multiscale decomposition of data,  with a projection on a reproducing kernel Hilbert space. The proposed approach, called kernel treelets (KT), effectively substitutes the correlation coefficient matrix used in treelets with a symmetric, positive semi-definite matrix efficiently constructed from a  kernel function. Unlike most clustering methods, which require data sets to be numeric, KT can be applied to more general data and yield a multi-resolution sequence of basis on the data directly in feature space.
The effectiveness and potential of KT in clustering analysis is illustrated with some examples. 
 \end{abstract}
\section{Introduction}
Treelets, introduced by Lee, Nadler, and Wasserman~\cite{Treelets,Treelets2}, is a method to produce a multiscale, hierarchical decomposition of unordered data. The central premise of Treelets is to exploit sparsity and capture intrinsic localized structures with only a few features, represented in terms of an orthonormal basis. The hierarchical tree constructed by the treelet algorithm provides a scale-based partition of the data that can be used for classification, specially for cluster analysis~\cite{clustering}. 

Cluster analysis, also called clustering, is concerned with finding a partition of a set such that its corresponding equivalence class captures similarity of its elements. The Treelet approach is an example of 
hierarchical clustering (HC) \cite{hierarchical_clustering}, which is a type of methods that provides a nested and multiscale clustering. The typical complexity of HC methods is $O(n^3)$ (where $n$ denotes the number of data in the dataset) but Treelets, like single linkage HC \cite{slink} and complete linkage HC \cite{clink}, 
can be done in $O(n^2)$ operations.  Most of these clustering methods are only applicable to numerical dataset only. However,  many modern datasets do not have clear representations in $\mathbb{R}^p$  due for example to missing data, length difference, and non-numeric attributes. A typical solution to this problem  usually involves finding a projection from each observation to $\mathbb{R}^p$ as is the case for example in text vectorization~\cite{vectorization}, array alignment~\cite{alignment}, and missing-data imputation~\cite{imputation}. These particular projections pose considerable challenges and might raise the bias of the model if false assumptions are made.

In this paper we propose a HC method that combines Treelets with  a projection on a feature space that is a
Reproducing Kernel  Hilbert Space (RKHS). We call this method Kernel Treelets (KT). It effectively substitutes the correlation coefficient matrix, used by the original treelet method as a measure of similarity among variables, with a symmetric, positive semi-definite matrix constructed from a (Mercer) kernel function. The intuition behind this approach is that inner products provide a measure of similarity  and a projection into a RKHS, done via the so-called Kernel trick~\cite{HastieBook,TheodoridisBook}, is a natural and efficient way to construct appropriate similarity matrices for a wide variety of data sets, including those mentioned above. We present some examples that demonstrate the potential of KT as an effective tool for clustering analysis. 


\section{Background Information}
We provide in this section a brief description of the Treelet algorithm~\cite{Treelets,Treelets2} and the Kernel method~\cite{kernel}. Treelets are based on the repeated application of two dimensional (Jacobi) rotations to a
matrix measuring the similarity of variables. So we start by reviewing Jacobi (also called Givens)  rotations first. 

\subsection{Jacobi Rotations}

A Jacobi rotation matrix $J$ is an orthogonal matrix with at most 4 entries different from the identity, or more generally, a rotation operator on a 2 dimensional subspace generated by two coordinate axes. 
For a given symmetric matrix $M$ and entry $pq$ and  rotation matrix $J$ is constructed so that 
$$[J^TMJ]_{pq} = [J^TMJ]_{qp} = 0.$$
The construction of $J$ is equivalent to finding the cosine (c) and sine (c) of the angle of rotation, which satisfy
\[\begin{bmatrix}c&-s\\s&c\end{bmatrix}\begin{bmatrix}M_{pp}&M_{pq}\\M_{qp}&M_{qq}\end{bmatrix}\begin{bmatrix}c&s\\-s&c\end{bmatrix}=\begin{bmatrix}d_1&0\\0&d_2\end{bmatrix}\]
subject to the constraint $c^2+s^2=1$. The matrix $J$ is then given
\begin{itemize}
\item  \(J_{pp} = J_{qq} = c\)
\item  \(J_{pq} = -J_{qp} = s\)
\item  For other entries \(ij\),  $J_{ij} = I_{ij}$.
\end{itemize}
A numerical stable way of computing this problem is as follows:
\begin{itemize}
\item Assume $M_{pq}\neq 0$, and compute $$b = \frac{M_{pp} - M_{qq}}{2M_{pq}}.$$
\item Let $sgn(b)$ be 1 if $b\geq 0$ and -1 otherwise, then we define
$$t = \frac{sgn(b)}{|b|+\sqrt{b^2+1}}.$$
\item From which we can calculate \(c = \frac{1}{\sqrt{t^2+1}}\) and \(s = ct\).
\end{itemize}
The complexity of storing a Given's rotation matrix is $O(1)$, and Jacobi rotation over a $n\times n$ matrix uses $O(1)$ space with time complexity $O(n)$.

\subsection{Treelets}

The Treelets algorithm~\cite{Treelets, Treelets2} was designed to construct a multiscale basis and a corresponding hierarchical clustering over the attributes of some datasets in $\mathbb{R}^p$, to exploit sparsity.  In its most efficient implementation~\cite{Treelets2} it is an $O(np^2)$ algorithm. 
The algorithm starts with a regularization, hyper-parameter $\lambda$ and computing a \(p\times p\) (empirical) covariance matrix \(A_0\).
The initial scaling indices are defined as the set
\(S_0 = \{1,2,...p\}\). With base case \(A_0\) and \(S_0\),
each step \(A_k\) and \(S_k\) for \(k\in \{1,2,3,...,p-1\}\) can be constructed inductively 
as follows:

\begin{enumerate}
\def\labelenumi{\arabic{enumi}.}
\item
  Construct matrix \(M_{k}\) of the same shape as \(A_0\) entry-wise:
  $$[M_k]_{ij} = \sqrt{\frac{[A_{k-1}]_{ij}^2}{[A_{k-1}]_{ii}[A_{k-1}]_{jj}}} + \lambda |[A_{k-1}]_{ij}|.$$
\item 
  Find the two indices $\alpha_k,\beta_k$ such that 
  $$\alpha_k, \beta_k = \argmax_{\alpha, \beta \in S_{k-1}} [M_k]_{\alpha\beta}.$$.
\item
  Calculate Jacobi rotation matrix \(J_k\) for $\alpha_k, \beta_k$ and matrix \(A_k = J_k^T A_{k-1} J_k\).
\item
  Without loss of generality, \(\alpha_k\) and \(\beta_k\) is interchangeable, so
  we require that \([A_k]_{\alpha_k\alpha_k}\leq [A_k]_{\beta_k\beta_k}\), and record
  \(\alpha_k\) and \(\beta_k\).
\item
  Define \(S_k = S_{k-1} - \{\alpha_k\}\).
\end{enumerate}

\subsubsection{Treelets Transform and Treelets Basis}

The Jacobi rotations produce a Treelets basis for each
\(k\in \{1,2,3,...,p-1\}\). The sequence of matrices \(\{J_k\}\)  provides a basis for \(\mathbb{R}^p\), 
defined as
\[B_k = J_k^TJ_{k-1}^T\cdots J_2^TJ_1^T,\]
such that 
\[A_k = B_k A_0 B_k^T.\] 
So for every vector \(v\in \mathbb{R}^p\), there
is a \(k\)th basis representation \(B_kv\). Furthermore, there is a
compressed \(k\)th basis representation obtained by dropping
insignificant (\(<\epsilon\)) non-scaling indices of \(B_kv\). That is,
if we define \(e_i\) to be the \(i\)th column of the identity matrix,
the compressed \(k\)th basis representation is given by
\[\tau_k(v) =B_kv - \sum_{\substack{i\not\in S_i\\|B_kv\cdot e_i|<\epsilon}}(B_kv\cdot e_i)e_i.\]

\subsubsection{Treelets Hierarchical Clustering}

Treelets is also a hierarchical clustering method over the attributes. The hierarchical clustering structure is stored in $\alpha_k, \beta_k$. We start with trivial clustering where each element is in its own cluster and labeled by itself. For each $k$, we merge clusters labeled $\alpha_k$ and $\beta_k$ and label it $\beta_k$. This is feasible because each step $k$ the set of all cluster labels is exactly $S_{k-1}$. This operation gives a hierarchical tree for clustering use on the attributes. 

\subsection{Kernel Method}

The Kernel method~\cite{kernel} allow us to map variables into a new feature space via a kernel function. We now review briefly the basic concepts and ideas of this approach (see for example~\cite{TheodoridisBook}).

A kernel over some set \(X\) is defined as a function
\(K:X\times X\to \mathbb{R}\). A symmetric and positive semi-definite (SPSD) kernel $K$ has the properties:
\begin{align}
K(x_1,x_2) &= K(x_2,x_1), \text{ for all $x_1, x_2\in X$}. \\
\sum_{i=1}^s\sum_{j=1}^s c_ic_jK(x_i,x_j) &\geq 0, \text{ for all $\{x_1,...,x_s\}\in X$ and  all $ \{c_1,...,c_s\}\in \mathbb{R}$}
\end{align}
If $X$ is finite, then $K$ is SPSD if and only if $K(X,X)$ is a SPSD matrix. 
If $X\subseteq \mathbb{R}^p$, then $K$ is SPSD if and only if there exists a function
\(\Phi_K:\mathbb{R}^p\to \mathbb{H}\), where \(\mathbb{H}\) denotes the
Hilbert space, such that for all \(x_1, x_2\in X\),
\begin{align}
K(x_1,x_2) = \langle\Phi_K(x),\Phi_K(y)\rangle_\mathbb{H}.
\end{align}
The space $\mathbb{H}$ here is called a reproducing kernel Hilbert space (RKHS). 
The following are two common examples of SPSD kernels:
\begin{enumerate}
    \item Radial basis function (RBF) kernel 
        \[K(x_1,x_2) = \exp \{-\frac{||x_1-x_2||^2}{2\sigma^2}\}.\]
    \item Polynomial kernel
        \[K(x_1,x_2) = (\alpha \langle x_1, x_2\rangle + c_0)^r.\]
\end{enumerate}
A kernel $K$ for a set $X$ can be restricted to a subset $Y\subseteq X$, and SPSD property is preserved during restriction. If the task is clustering over a finite set, the selected kernel needs only be SPSD on the set of all samples, which is generally finite, and we only need to check that the kernel matrix is SPSD. If we need to extend the clustering outcome to other data, e.g. clustering boosted classification, then $X$ has to include the whole data space as a subset. 

\subsection{K Nearest Neighbors (KNN)}

K-nearest neighbors algorithm is a multi-class classification algorithm \cite{knn}. By specifying $k\in \mathbb{N}$ and a metric, the algorithm can, given a test data, predict its labels by the majority vote of a subset of $k$ closest elements in distance metric from training data. If an inner product is specified instead of distance, we can compute the distance between two point in the following way:
$$\lVert x_1 - x_2\rVert ^2 = \langle x_1 - x_2, x_1 - x_2\rangle = \langle x_1, x_1\rangle + \langle x_2, x_2\rangle -2\langle x_1, x_2\rangle .$$
If the metric is kernelized, 
\begin{align*}
    \lVert x_1 - x_2\rVert ^2 
    &= \langle x_1, x_1\rangle_\mathbb{H} + \langle x_2, x_2\rangle_\mathbb{H} -2\langle x_1, x_2\rangle_\mathbb{H}
    \\&= K( x_1, x_1) + K( x_2, x_2) -2K( x_1, x_2).
\end{align*}

\subsection{Kernel Support Vector Machine (SVM)}

Support Vector Machine (SVM) is a classification method by finding optimal hyper-planes. Kernel SVM \cite{ksvm} is a classification method towards nonlinear problems that performs SVM in RKHS generated by the kernel. When we only apply KT to a small sample, we may use kernel SVM with the same kernel to assign labels for data outside of this sample. This can be viewed as clustering attributes with treelets and using SVM to assign labels to other attributes in RKHS.

\section{The KT Model}

The task of KT is to find a clustering for some set $X$ given a SPSD kernel $K:X\times X\to \mathbb{R}$ measuring the similarity among variables. We combine Treelets with kernels by replacing the covariance $A_0$ with kernel matrix, and apply the rest of the steps of Treelets algorithm. The exact steps are as follows:

\begin{enumerate}
    \item First we draw a sample $S$ with size $n_S$ from uniform distribution on $X$ and some sample size $n_X$. If more information about $X$ is given, it may be possible to draw a sample $S$ that better represent $X$ with smaller sample size. 
    
    \item Then, we calculate the kernel matrix $A_0 = K(S,S)$. $A_0$ is a SPSD matrix because $K$ is SPSD, and thus we can apply Treelets algorithm with hyper-parameter $\lambda$ using $A_0$ instead of the (empirical) covariance matrix. $\lambda$ can be set to 0 or tuned experimentally as in Treelets. In this step,  theTreelets method provides a hierarchical clustering tree of each columns of $A_0$, which corresponds to each observation in $S$. 
    
    \item If $S=X$, we are finished on the step above. Otherwise, we need to cluster the elements in $X$ based on clusters we have from elements in $S$. We use kernel SVM to complete this task. Given $S$ and its corresponding cluster labels, we train the kernel SVM with the same kernel $K$, and then apply to predict the cluster labels of $X$. K-Nearest Neighbors with distance induced by kernel
    $$d(v_1, v_2)^2 = K(v_1, v_1) + K(v_2, v_2) - 2K(v_1, v_2)$$
    is an alternative to kernel SVM. 
\end{enumerate}

\subsection{Theory}
We now prove that the kernel projection is equivalent to working with a symmetric positive definite matrix 
defined by the inner product in $\mathbb{H}$ and evaluated through the kernel. We also suggest a definition of a clustering setting and clustering equivalence that allows us to connect the results of the clustering analysis for the original set with those of the transformed, projected set. 
\begin{lemma}
For every finite dataset $D = \{d_i: i=1,2,...,n\} \subseteq X$ and an SPSD kernel $K$, there exists an orthonormal Hilbert basis $B$ in the RKHS such that 
$$[\Phi_K (d_i)]_B = \begin{bmatrix}\delta_i\\0\end{bmatrix}, $$ 
where $\delta_i\in \mathbb{R}^n$ and $\begin{bmatrix} \delta_1 & \delta_2 & \cdots & \delta_n \end{bmatrix}$
is symmetric and positive semi-definite. 
\end{lemma}
\begin{proof}
We apply Gram-Schmidt orthogonalization process to the maximal linearly independent subset of  $\{\Phi_K(d_i):i=1,2,...,n\}$ and get a set of orthonormal vectors $\{\hat\beta_i:i=1,2,...,\eta\}$, where 
$$\eta = \text{dim}(\text{span}\{\Phi_K(d_i):i=1,2,...,n\}) \leq n.$$
We may extend this set to a orthonormal Hilbert basis $\hat B = \{\hat\beta_i:i=1,2,...\}$. Then $\forall i\in \{1,2,...,n\}$, $[\Phi_K(d_i)]_{\hat B}$ is 0 for all entries after $\eta$ and consequently after $n$, so there exists $\hat d_i \in \mathbb{R}^n$ such that
$$[\Phi_K(d_i)]_{\hat B} = \begin{bmatrix} \hat{d}_i \\ 0 \end{bmatrix}.$$
As $\begin{bmatrix} \hat d_1 & \hat d_2 & \cdots & \hat d_n\end{bmatrix}$ is a square matrix, we may compute its singular value decomposition 
$$\begin{bmatrix} \hat d_1 & \hat d_2 & \cdots & \hat d_n\end{bmatrix} = U\Sigma V^T.$$
We can now define a new orthonormal Hilbert basis $B = \{\beta_i:i=1,2,...\}$ through  the change of basis matrix $\begin{bmatrix} VU^T & 0 \\ 0 & I\end{bmatrix}$. 
Let $\delta_i = VU^T\hat d_i$ for all $i\in \{1,2,...,n\}$, then
$$[\Phi_K(d_i)]_B = \begin{bmatrix} VU^T & 0 \\ 0 & I\end{bmatrix}[\Phi_K(d_i)]_{\hat B} = \begin{bmatrix}VU^T\hat d_i\\ 0\end{bmatrix} = \begin{bmatrix}\delta_i \\ 0\end{bmatrix}.$$
The projected data $\Phi_K(d_i)$ in basis $B$ is $\begin{bmatrix}\delta_i^T & 0\end{bmatrix}^T$ and  the matrix
$$\begin{bmatrix} \delta_1 & \delta_2 & \cdots & \delta_n \end{bmatrix} = QU^T\begin{bmatrix} \hat d_1 & \hat d_2 & \cdots & \hat d_n\end{bmatrix} = Q\Sigma Q^T$$
is symmetric and positive definite. 
\end{proof}
\begin{corollary}
If we denote $\Psi: V\to \mathbb{R}^n$ such that for all $v\in V$, 
$$\begin{bmatrix}\Psi(v)\\ *\end{bmatrix} = [\Phi_K(v)]_B.$$
or in other words, $\Psi(v)$ is the first $n$ components of $\Phi_K(v)$ in the basis $B$. Then for all $d_i\in D$, 
$$\begin{bmatrix}\Psi(d_i)\\0\end{bmatrix} = [\Phi_K(d_i)]_B,$$
that is $\Psi(d_i) = \delta_i$. From the lemma, we have that
$\Psi(D) = \begin{bmatrix}\delta_1&\delta_2&\cdots &\delta_n\end{bmatrix}$
is symmetric and positive definite and 
$$\langle \Psi(D), \Psi(D)\rangle = \begin{bmatrix}\delta_1&\delta_2&\cdots &\delta_n\end{bmatrix}^2 = \langle \Phi_K(D), \Phi_K(D)\rangle_\mathbb{H}.$$
\end{corollary}
\subsubsection{Clustering Equivalences}

A clustering setting is a pair $(D, f)$ where $D$ is an finite ordered dataset and $f:D\times D\to \mathbb{R}$ is a measurement on the dataset $D$. We define an equivalence on the clustering setting that $(D_1, f_1) = (D_2, f_2)$ if and only if $f_1(D_1, D_1) = f_2(D_2, D_2)$. For any measurement based clustering method, using measurement $f_1$ on $D_1$ provides the same exact clustering outcome on the labels as using measurement $f_2$ on $D_2$. An example of clustering equivalences is that if kernel $K$ corresponds to projection $\Phi_K$, then there is $K(D, D) = \langle \Phi_K(D), \Phi_K(D)\rangle_\mathbb{H}$, and therefore $(D, K) = (\phi_K(D), \langle \cdot, \cdot \rangle_\mathbb{H})$. 

\subsubsection{Kernel Treelets Equivalences}

For a dataset $\{d_i:i=1,2,...,n\}$ and a kernel $K$, we already know that there is a clustering equivalence 
$(D, K) = (\phi_K(D), \langle \cdot, \cdot \rangle_\mathbb{H})$. 
From the corollary of lemma 1, there is 
$\langle \Psi(D), \Psi(D)\rangle = \langle \phi_K(D), \phi_K(D)\rangle_\mathbb{H}$, 
which provides the equivalence 
$(\phi_K(D), \langle \cdot, \cdot \rangle_\mathbb{H}) = (\Psi(D), \langle \cdot, \cdot \rangle)$. As $\Psi(D)$ is symmetric, 
$(\Psi(D), \langle \cdot, \cdot \rangle) = (\Psi^T(D), \langle \cdot, \cdot \rangle)$.
As a conclusion, $(D, K) = (\Psi^T(D), \langle \cdot, \cdot \rangle)$, which implies that a clustering method measured with inner product on dataset $\Psi^T(D)$ provides a clustering of $D$ measured with kernel $K$. Therefore, Treelets on $\Psi(D)$ without centering provides a hierarchical clustering of attributes of $\Psi(D)$ based on attribute inner product (covariance matrix), which is a hierarchical clustering of $\Psi^T(D)$ based on inner product. According to clustering setting equivalences, this hierarchical clustering is equivalent to a hierarchical clustering of $D$ based on kernel $K$. Furthermore, a property of Treelets is that $\Psi(D)$ does not necessarily need to be computed. The "covariance matrix" of $\Psi(D)$ without centering has a easier computation method:
$$Cov(\Psi(D)) = \Psi(D)\Psi(D)^T = \Psi(D)^2 = \langle \Psi(D), \Psi(D)\rangle = \langle \phi_K(D), \phi_K(D)\rangle_\mathbb{H} = K(D, D).$$
So we may avoid the costly spectral decomposition to compute $\Psi(D)$ and define $A_0$ of Treelets as 
$$A_0 = Cov(\Psi(D)) = K(D, D).$$

\subsection{Complexity}

The complexity of this algorithm is $O(\xi n_S^2 + n_S n_V)$, where $\xi$ is the complexity of applying kernel function to a pair of data and $\xi=p$ if the data is numeric. In this model, the choice of kernel $K$ determines the expected outcome of the prediction and the choice of sample $S$ determines the variability of the outcome. A small sample size $S$ speeds up the algorithm with the cost of generating false clustering by unrepresentative samples, while large sample size slow down the algorithm and also produces numerical issues because data is more likely to be close to orthogonal as the dimension of projected space grows, and Treelets method would be forced to stop if all remaining components are almost orthogonal. The optimal sample size depends on the floating number accuracy and computation time allowed and should be as large as possible without exceeding the time limit and accuracy limit.
\section{Examples}

We implemented KT and the following examples in Python  with package Numpy~\cite{numpy}, Scikit-learn~\cite{sklearn}, and plots were generated with  Matplotlib~\cite{matplotlib}. The Treelets part of our implementation is not optimized, so it is $O(n^3)$ runtime in the followings examples rather than $O(n^2)$ as designed by Lee et al \cite{Treelets}. The hyperparameter $\lambda$ is set to 0 for all the experiments below.

\subsection{Clustering for 6 Datasets}

To illustrate how KT works as a hierarchical clustering method, 
we use an example from scikit-learn~\cite{sklearn} which consists of 6 datasets, each of which has 1500 two-dimensional data points (i.e. $n=1500$ and $p=2$), and we can visualize each dataset and each cluster by plotting each observation as a point in the plane. Each of the first five datasets consists of data drawn from multiple shapes with an error in distance. The sixth dataset consists of a uniform random sample from $[0,1]^2$ to show how clustering method work for uniform distributed data.  Figure~1 shows how KT with different kernels works on these datasets compared to the performance of some other clustering methods. The number of clusters and hyper-parameters are tuned for each method and the sample sizes are set to 1000 for each KT method. Each row of this image represents a dataset and each column represents a clustering method. The method each column represents and and its runtime on each dataset is in recorded in Table~1. 

\begin{figure}[ht!]\centering
\includegraphics[width=1\textwidth]{{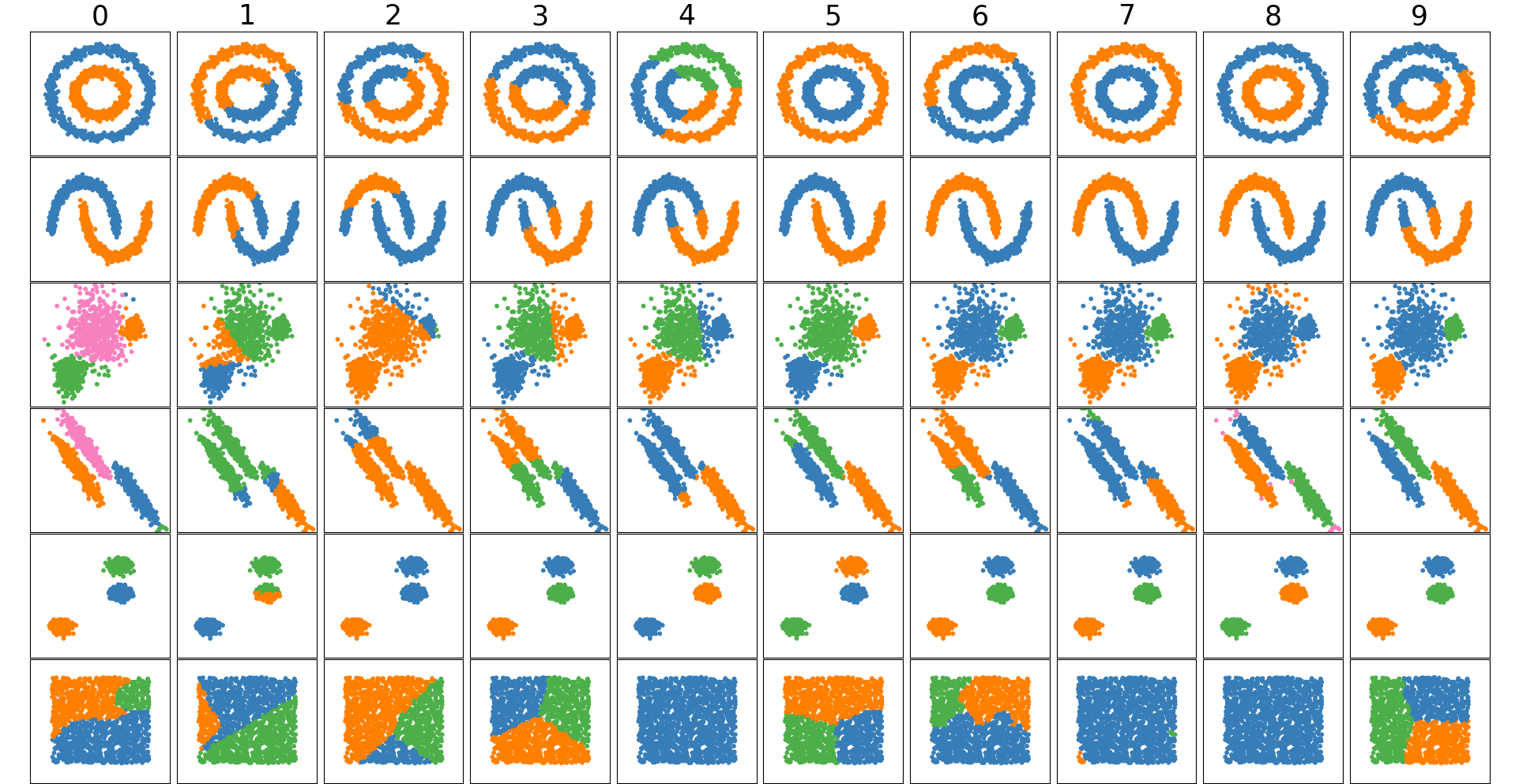}}
\caption{Comparison of different clustering algorithms on 6 datasets.}
\end{figure}
\begin{table}[ht!]
\begin{center}
\begin{tabular}{|l|l|l|l|l|l|l|}
\hline
Method\textbackslash{}Dataset & 1     & 2     & 3     & 4     & 5     & 6     \\ \hline
0 - KTrbf                     & 2.003 & 2.063 & 2.325 & 2.094 & 2.819 & 1.967 \\ \hline
1 - KTlinear                  & 1.585 & 1.613 & 1.402 & 1.73  & 2.341 & 1.469 \\ \hline
2 - KTpoly                    & 3.956 & 6.08  & 6.878 & 9.582 & 9.836 & 4.526 \\ \hline
3 - MiniBatchKMeans           & 0.006 & 0.018 & 0.009 & 0.01  & 0.007 & 0.009 \\ \hline
4 - MeanShift                 & 0.047 & 0.032 & 0.063 & 0.057 & 0.032 & 0.05  \\ \hline
5 - SpectralClustering        & 0.642 & 1.011 & 0.13  & 0.352 & 0.257 & 0.208 \\ \hline
6 - Ward                      & 0.114 & 0.098 & 0.513 & 0.245 & 0.111 & 0.087 \\ \hline
7 - AgglomerateClustering     & 0.085 & 0.102 & 0.374 & 0.196 & 0.103 & 0.078 \\ \hline
8 - DBSCAN                    & 0.015 & 0.014 & 0.015 & 0.012 & 0.067 & 0.012 \\ \hline
9 - GaussianMixture           & 0.005 & 0.005 & 0.008 & 0.012 & 0.004 & 0.009 \\ \hline
\end{tabular}
\end{center}
\caption {Method and Runtime Table for Figure 1}
\end{table}

In this experiment, KT with RBF kernel is the only method that performs clustering closest to human intuition for all first five datasets. The sixth dataset is a uniform distribution in $[0,1]^2$ which we may see how KT is affected by the relative density deficiency in some area due to sampling. Its high performance on the first five datasets is expected as these datasets are to some extent Euclidean distance-based, which corresponds to the assumptions for RBF kernels. Fig.2 shows how difference of number of sample points affects the clustering result. Each column represents KT using RBF kernel with different sample sizes. The hyper-parameter $\sigma=0.1$ is tuned towards $n_S=1000$ case and is used for all other sample sizes. Notice that as KT1500 is of full sample size, it does not trigger kernel SVM whereas KT1499 do. Their number of clusters and runtime is recorded in Table 2. From here we can see that more sample data implies more runtime and more stable outcome. The minimum optimal number of samples required for the first 5 datasets are 1000, 100, 1000, 200, 50, respectively, which shows that different datasets requires different amount of samples to explain its shape. Furthermore, the fourth dataset shows that optimal hyper-parameter $\sigma$ is number-of-sample dependent. RBF kernel can be considered as a weighted average of distance and connectivity, where a larger $\sigma$ means a higher weight on distance. For the same $\sigma=0.1$, as sample size gets larger, the clustering result becomes more distance-based rather than connectivity based, demonstrating that optimal $\sigma$ for those sample sizes are actually smaller.

\begin{figure}[ht!]\centering
\includegraphics[width=1\textwidth]{{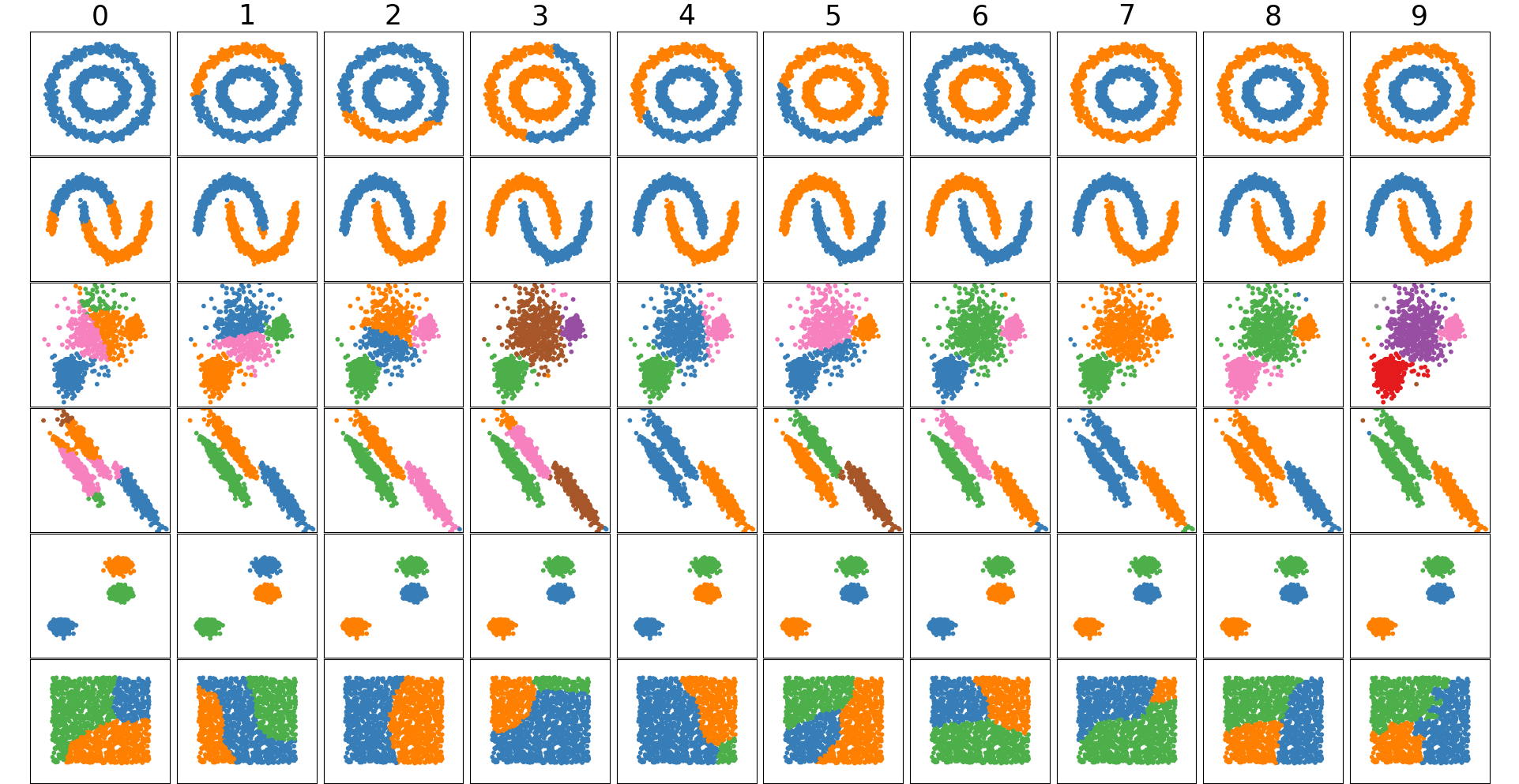}}
\caption{Comparison of different number-of-cluster estimate on 6 datasets.}
\end{figure}
\begin{table}[ht!]
\begin{center}
\begin{tabular}{|l|l|l|l|l|l|l|}
\hline
Method\textbackslash{}Dataset & 1     & 2     & 3     & 4     & 5     & 6     \\ \hline
0 - KT50                      & 0.011 & 0.012 & 0.013 & 0.011 & 0.011 & 0.01  \\ \hline
1 - KT100                     & 0.035 & 0.044 & 0.039 & 0.045 & 0.033 & 0.028 \\ \hline
2 - KT200                     & 0.109 & 0.099 & 0.128 & 0.12  & 0.121 & 0.132 \\ \hline
3 - KT300                     & 0.225 & 0.217 & 0.242 & 0.269 & 0.259 & 0.235 \\ \hline
4 - KT500                     & 0.551 & 0.568 & 0.62  & 0.569 & 0.652 & 0.536 \\ \hline
5 - KT800                     & 1.315 & 1.513 & 1.534 & 1.378 & 1.699 & 1.295 \\ \hline
6 - KT1000                    & 2.016 & 2.055 & 2.336 & 2.098 & 2.782 & 1.941 \\ \hline
7 - KT1200                    & 2.88  & 2.94  & 3.242 & 3.004 & 4.146 & 2.77  \\ \hline
8 - KT1499                    & 4.438 & 4.532 & 5.4   & 4.713 & 6.788 & 4.341 \\ \hline
9 - KT1500                    & 4.472 & 4.69  & 5.398 & 4.807 & 6.782 & 4.274 \\ \hline
\end{tabular}
\end{center}
\caption {Method and Runtime Table for Figure 2}
\end{table}

\subsection{Clustering for Social Network Dataset}

To illustrate how KT works in network analysis
we use an example from Stanford Network Analysis Project~ \cite{network_data}. 
This is a dataset consisting of 'circles' (or 'friends lists') from Facebook. 
It has $n_V=4039$ surveyed individual (vertices) and each two of them is connected with vertices if they are friends and not if they are not friends, which are the edges. 
The edges are undirected and not weighted, and the total number of edges is 88234. We use KT to do clustering on this dataset with full sample size ($S = V$). Denote the set of vertices on the graph as $V$, and define a kernel function $K:V\times V\to \mathbb{R}$ such that 

$$K(v_1,v_2) = \begin{cases}
1045 & v_1=v_2\\
1 & v_1,v_2\ are\ connected\\
0 & otherwise
\end{cases}$$

The number 1045 is computed and chosen as the largest degree of all vertices. Notice that $K$ is a SPSD kernel on $V$ because $K(V,V)$ is a symmetric matrix and is also dominant by the positive diagonal, as $\forall i\in \{1,2,...,n\}$
$$\sum_{j\neq i} |K(V,V)_{i,j}| = deg(v_i) \leq \max_\zeta deg(v_\zeta) = 1045 = K(V,V)_{i,i}.$$
To estimate the performance of KT as a multi-scale clustering method on this dataset, we use an evaluation as follows. For each cluster partition in the hierarchy, we compute its matching matrix and its corresponding true positive rate as well as false positive rate. Matching matrix, a type of confusion matrix, is a 2 by 2 matrix recording the number of true positives, true negatives, false positives, and false negatives for pairwise associations. True positive rate measure the proportion of two nodes being in the same cluster given the two nodes are connected and false positive rate measures the proportion of two nodes being in the same cluster given the two nodes are not connected. Each pair of true positive rate and false positive rate produces a point on the plane, and interpolating the set of points of all clustering results in the hierarchy (with order) produces the Receiver operating characteristic (ROC) curve, and the numerical integral over $[0,1]$ interval of this curve is known as Area Under Curve (AUC). Figure~3 demonstrates the performance of KT on the dataset, which provides good clusterings for the dataset because it has an AUC as high as $0.958$. 

\begin{figure}[ht!]\centering
\includegraphics[width=1\textwidth]{{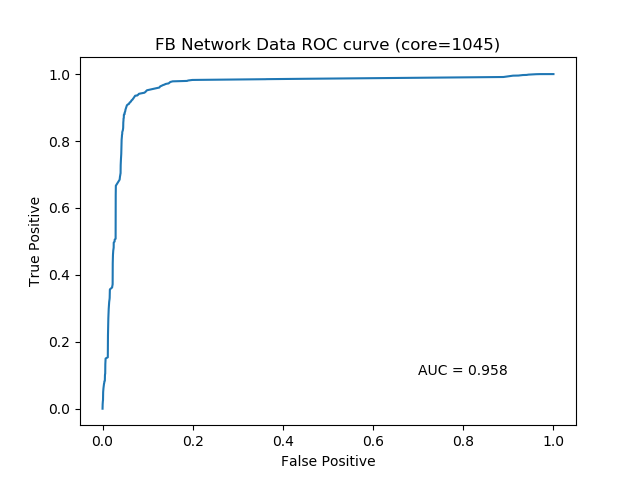}}
\caption{Clustering of Facebook Network Dataset Result}
\end{figure}

\subsection{Clustering for Dataset with Missing Infomation}

To illustrate how KT works on dataset with missing information,
we use Mice Protein Expression (MPE) dataset \cite{mice} 
from UCI Machine Learning Repository as an example. 
This is a dataset consisting of 1080 observations for 8 classes of mice, 
each of which containing 77 expression levels of different proteins 
with some of the entries are not avalible. 
We use KT to do clustering on this dataset. First we normalize these attributes so that each of them has empirical mean 0 and standard deviation 1. Then we define a RBF kernel for dataset with missing data such that for all observation $u, v$, 
$$K(u,v) = \exp \Bigg\{ - \frac{32}{|E_{uv}|} \sum_{i\in E_{uv}}\lVert u_i-v_i\rVert^2\Bigg\}$$
Where $E_{uv}$ is the set of indices that is avalible (not missing) in both $u$ and $v$. 
We check that $E_{uv}\neq \emptyset$ so that it is well-defined.
The number 32 is a parameter tuned with experiments. 
We compare the predicted clusters and the true labels according to pairwise scores. Fig.4 shows how KT performs compared to KMeans clustering. We measure the true positive rate as the proportion of two record being in the same cluster given that they are from mice of the same type, and the false positive rate as the proportion of two record being in the same cluster given that they are from mice of different type. Similar as the example of network dataset, we draw its ROC curve and calculate its AUC. Also, we use KMeans with multiple number of clusters for comparison. The AUC of KT is much higher than the AUC of KMeans ($0.726 > 0.579$), demonstrating KT is a much better clustering method for this dataset than KMeans.

\begin{figure}[ht!]\centering
\includegraphics[width=1\textwidth]{{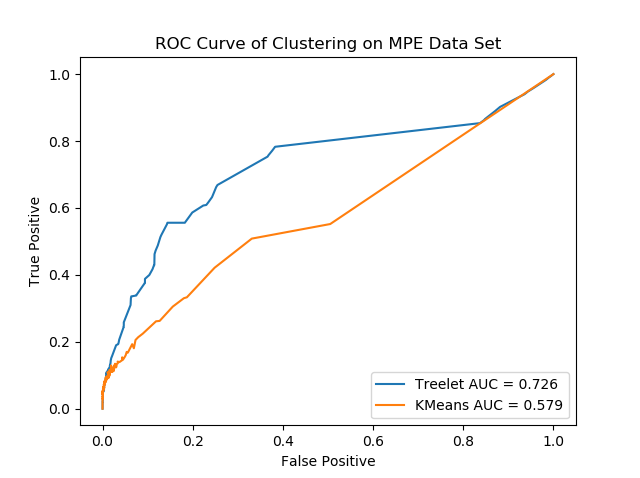}}
\caption{Comparison of KT and KMeans on MPE dataset}
\end{figure}

\section{Conclusion}

In the paper we describe a novel approach, kernel treelets (KT),  for hierarchical clustering. The method relies on applying the treelet algorithm to a matrix measuring similarities among variables in a feature, reproducing kernel Hilbert space. We show with some examples that KT is as useful as other hierarchical clustering methods and is especially competitive for datasets  without numerical matrix representation and or missing data. The KT approach also shows significant potential for semi-supervised learning tasks and as a pre-processing, post-processing step in deep-learning. Work in these directions is underway.

\medskip

\bibliographystyle{unsrt}
\bibliography{main}

\end{document}